\documentclass[letterpaper]{article} 
\usepackage{aaai25}  
\usepackage{times}  
\usepackage{helvet}  
\usepackage{courier}  
\usepackage[hyphens]{url}  
\usepackage{graphicx} 
\usepackage{natbib}  
\usepackage{caption} 
\usepackage{algorithm}
\usepackage{algorithmic}

\usepackage{newfloat}
\usepackage{listings}

\usepackage{multirow}
\usepackage{amsmath}
\usepackage{amssymb} 
\usepackage{bm}
\usepackage{graphicx} 
\usepackage{float}
\usepackage{amsthm}
\usepackage{makecell}
\usepackage{placeins}
\newtheorem{proposition}{Proposition}
\newtheorem{theorem}{Theorem}

\DeclareCaptionStyle{ruled}{labelfont=normalfont,labelsep=colon,strut=off} 
\floatstyle{ruled}
\newfloat{listing}{tb}{lst}{}
\floatname{listing}{Listing}
\title{GNN-Transformer Cooperative Architecture for \\Trustworthy Graph Contrastive Learning}
\author{
    Jianqing Liang,
    Xinkai Wei,
    Min Chen,
    Zhiqiang Wang,
    Jiye Liang\thanks{Corresponding author.} 
}
\affiliations{
    Key Laboratory of Computational Intelligence and Chinese Information Processing of Ministry of Education, 
    School of Computer and Information Technology, Shanxi University,
    Taiyuan 030006, Shanxi, China \\
    liangjq@sxu.edu.cn, \{202322407046,202222409003\}@email.sxu.edu.cn, \{wangzq,ljy\}@sxu.edu.cn
}

\usepackage{bibentry}

\begin{document}

\maketitle

\begin{abstract}
Graph contrastive learning (GCL) has become a hot topic in the field of graph representation learning. In contrast to traditional 
supervised learning relying on a large number of labels, GCL exploits augmentation strategies to generate multiple 
views and positive/negative pairs, both of which greatly influence the performance. Unfortunately, commonly used random augmentations may 
disturb the underlying semantics of graphs. Moreover, traditional GNNs, a type of widely employed encoders in GCL, 
are inevitably confronted with over-smoothing and over-squashing problems. To address these issues, 
we propose GNN-Transformer Cooperative Architecture for Trustworthy Graph Contrastive Learning (GTCA), which inherits the advantages of 
both GNN and Transformer, incorporating graph topology to obtain comprehensive graph representations.
Theoretical analysis verifies the trustworthiness of the proposed method.
Extensive experiments on benchmark datasets demonstrate state-of-the-art empirical performance.
\end{abstract}

%
\begin{links}
    \link{Code}{https://github.com/a-hou/GTCA}
\end{links}

\section{Introduction}

Compared with traditional supervised learning, self-supervised learning eliminates the dependence on labels.
As one of the most representative methods of self-supervised learning, contrastive learning has been widely applied in the graph domain, i.e., Graph Contrastive Learning (GCL).
GCL utilizes multiple views and positive/negative pairs for node/graph representations \cite{GRACE,GMI,zhu2021contrastive,yin2022autogcl,ijcai_subgraph},
the performance of which are influenced by augmentation strategies, graph encoders and loss functions.

Different from computer vision domain where augmented images obtained with cropping, rotation, and other strategies usually retain the same semantic as the original images \cite{shorten2019survey},
graph augmentation will disturb the underlying semantics of graphs, which may affect the performance on downstream tasks \cite{li2022let}.
Figure~\ref{fig:example} (a) shows that even after randomly cropping or changing the color of an image, the underlying semantic still remains the same, which we can still easily recognize.
Figure~\ref{fig:example} (b) presents that removal of an atom or edge from the molecules may change their structures, resulting in different compounds.
\begin{figure}
\centering
    \includegraphics[height=0.55\columnwidth]{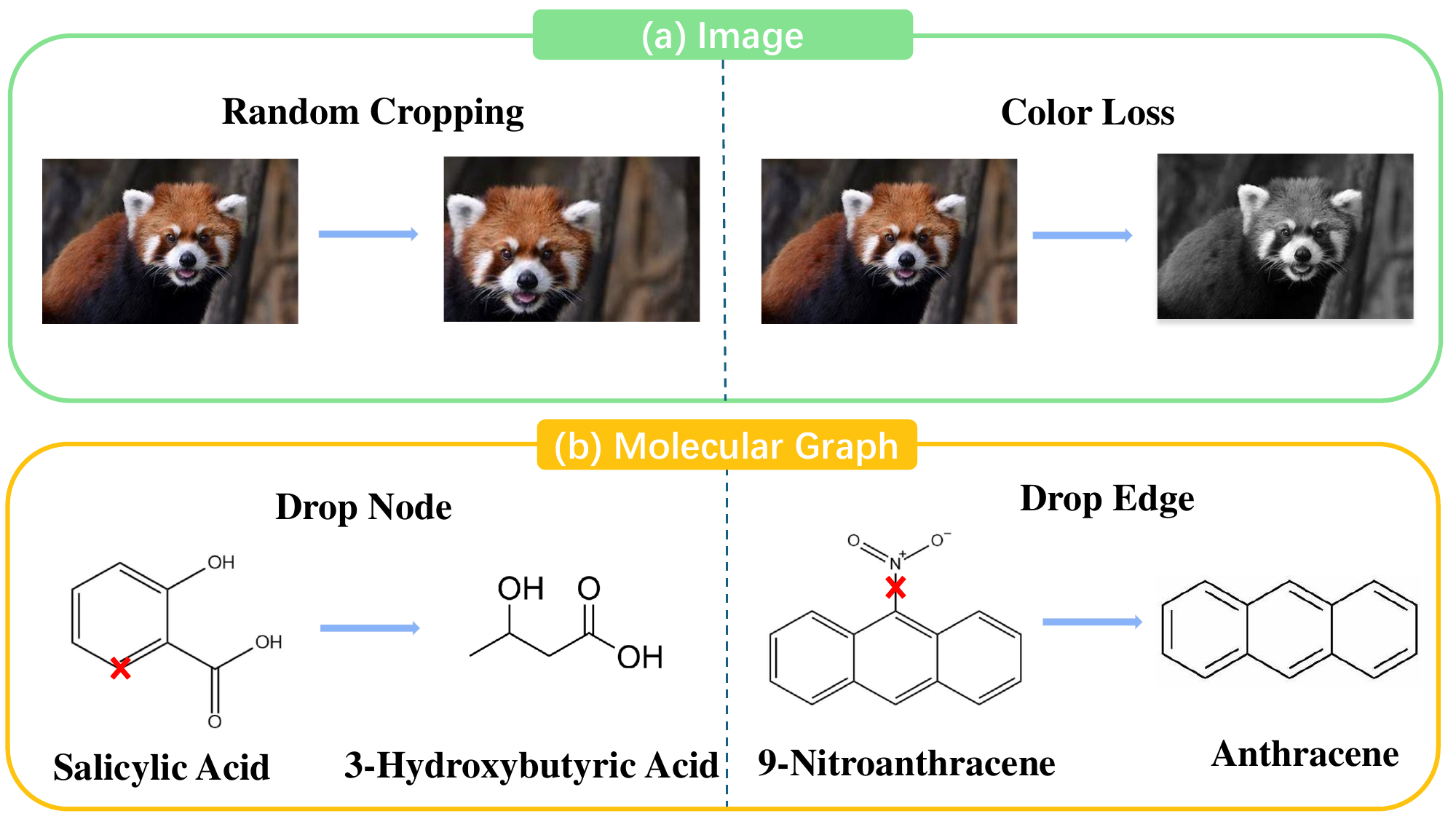}
    \caption{Augmentations on image (a) keep the semantic information, while augmentations on graphs (b) change the underlying semantic information.}
    \label{fig:example}
    \vspace{-1ex} 
\end{figure}
In order to address this issue, a number of augmentation strategies on graphs spring up. 
GCA~\cite{GCA} introduces learnable adaptive augmentation strategies with the structural and attribute characteristics of graphs for GCL.
GCS~\cite{wei2023boosting} employs gradient-based strategies to differentiate between semantics and environment within graphs. 
It perturbs the environment while maintaining the semantics to obtain positive pairs, and perturbs the semantics while preserving the environment to obtain negative pairs.
However, effective augmentation strategies must achieve a balance between the preservation of graph structural information and the generation of meaningful views \cite{suresh2021adversarial}. 
Poorly designed augmentations may result in redundant or irrelevant information generation, leading to inferior performance.

Most of the existing GCL methods exploit GNNs as graph encoders \cite{DGI,MVGRL,SUGRL,AFGRL,NCLA}.
GNNs iteratively propagate, transform, and aggregate representations from topological neighbors to progressively update node features~\cite{kipf2016semi,hamilton2017inductive,zhang2022graph}.  
With multi-layer stacking, these local features are progressively integrated to form a comprehensive perception of the entire graph structure. 
While this mechanism is particularly effective in handling local dependencies, especially in capturing both direct and indirect relation between nodes,
it fails to capture complex global structures and long-range dependencies 
within graphs \cite{xu2018powerful,garg2020generalization}.
In contrast, Graph Transformers (GTs) are capable of effectively capturing global dependencies between nodes with self-attention mechanisms~\cite{GAT,yun2019graph,kreuzer2021rethinking}. 
GTs can handle long-range dependencies within graphs, rather than being confined to local neighborhoods. 
By calculating attention scores between nodes, GTs update node representations with all other nodes in the graph, thereby effectively integrating global information. 
This mechanism enables GTs to model complex interactions and long-term dependencies between nodes.
Therefore, they are more flexible and accurate in handling dependencies between distant nodes \cite{yun2019graph},
which is ignored by existing GCL methods in choosing graph encoders to learn node/graph representations.
However, the quadratic complexity of GTs limits the scalability. 
Currently, some works introduce linear attention  mechanisms to improve scalability for large-scale graphs~\cite{nodeformer,SGFormer,liang}. 

Existing GCL methods mostly use InfoNCE as their loss functions~\cite{GRACE,zhang2022costa,spectral}, where each anchor point has only one positive pair. 
Therefore, when calculating the loss, nodes from the same class as the anchor point are pushed away. 
This may lead to ignorance of the potential information from similar nodes when handling node feature similarity tasks~\cite{NCLA}. 
While this design effectively encourages the model to distinguish between positive and negative pairs, it may lead to insufficient aggregation of nodes from the same class, 
thereby affecting the ability to represent graph data and capture graph structures.

In order to address the issues above, we propose a novel method called GNN-Transformer Cooperative Architecture for Trustworthy Graph Contrastive Learning (GTCA). 
GTCA utilizes GCN and NodeFormer, a linear GT as graph encoders to generate node representation views without random augmentation strategies.
In addition, it utilizes topological property of graphs to generate topology structure view. 
Furthermore, we design a novel loss function to exploit the intersection of the node representation views and topology structure view
as positive pairs.
The contributions are summarized follows:
\begin{itemize}
\item We utilize both GCN and NodeFormer as graph encoders to capture comprehensive perception of graphs, which has not been well explored in 
the field of GCL to our best knowledge.
\item We generate topology structure view with topological property of graphs.
Therefore, we introduce an augmentation-free strategy for GCL, which can enhance efficiency and avoid the potential risks of disturbing the underlying semantics of graphs. 
\item We design a novel contrastive loss function with multiple positive pairs for each node. Theoretical analysis and experimental results demonstrate the effectiveness of the GTCA method.
\end{itemize}

\section{Related Work}



\noindent\textbf{Graph Contrastive Learning.} Current GCL methods can be categorized into three mainstream paradigms: DGI framework~\cite{DGI}, InfoNCE framework~\cite{InfoNCE} and BGRL framework~\cite{bgrl}.

DGI aggregates the features of all nodes in the graph to obtain a global feature, 
then maximizes the mutual information between the global feature and the node features to learn node representations.
Based on this framework, MVGRL \cite{MVGRL} further advances the development of unsupervised graph contrastive learning through the adoption of graph diffusion and subgraph sampling methods.
In spite of the competitive performance, global features in these methods may be insufficient to retain node-level embedding information. 

GRACE \cite{GRACE} first utilizes the InfoNCE loss to maximize the mutual information between positive pairs under two augmented views, while minimizing the mutual information between negative pairs to learn node representations.
This principle is widely used in both node classification~\cite{GCA,spectral,NCLA} and graph classification~\cite{you2020graph,MVGRL,wei2023boosting} tasks.
While these methods achieve significant success, they still suffer from sampling bias issues.

BGRL uses the BYOL \cite{grill2020bootstrap} method to remove negative pairs, thereby reducing computational complexity. 
However, it relies on graph augmentation to obtain positive pairs, which may disturb the underlying semantics of the graph.
AFGRL~\cite{AFGRL} extends BGRL~\cite{bgrl} with an augmentation-free strategy and utilizes \( k \)-means for positive pair sampling. 
However, the randomness of \( k \)-means may influence the performance of down-stream tasks.

\noindent\textbf{Graph Augmentation.} 
The existing graph augmentation strategies include node dropping \cite{you2020graph}, edge perturbation \cite{qiu2020gcc,spectral}, attribute masking \cite{GCA,zhang2022costa} and subgraph extraction \cite{MVGRL}.
GRACE \cite{GRACE} employs random edge perturbation and node feature masking to generate two views.
GCA \cite{GCA} proposes an adaptive augmentation strategy which integrates both structural and attribute information.
Similarly, GCS \cite{wei2023boosting} utilizes the structural and semantic information to achieve adaptive graph augmentation.
Furthermore, CI-GCL \cite{tan2024community} proposes a community invariant GCL framework to maintain graph community structure.
Despite the considerable success, bias is introduced due to the disturbance of the semantic.

\begin{figure*}
	\centering
	\includegraphics[width=0.9\linewidth]{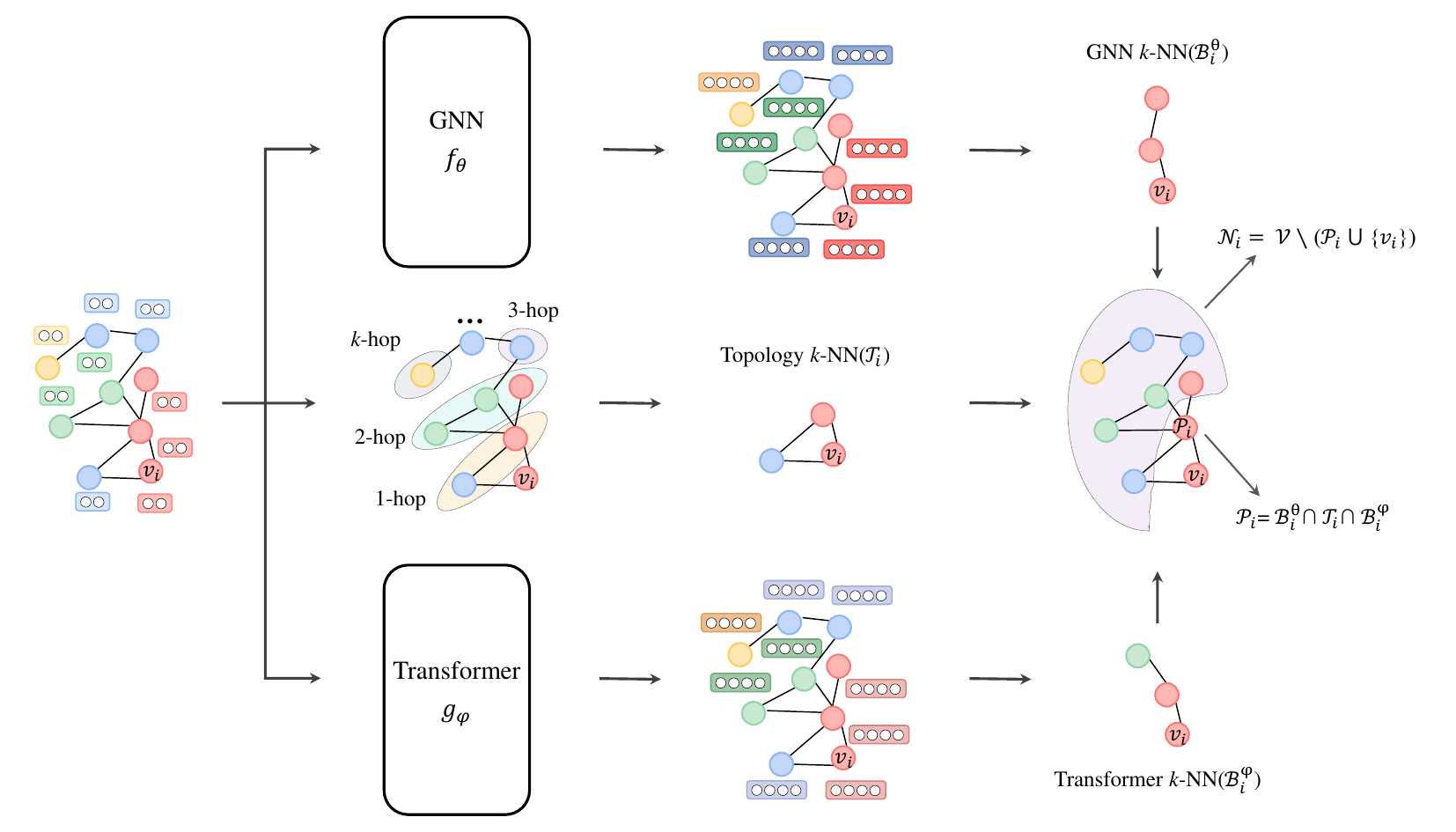} 
	\caption{The architecture of GTCA. Given a graph, the GNN encoder \( f_\theta \) and Transformer encoder \( g_\varphi \) generate node embeddings \(\boldsymbol{H}_{\theta}\) and \(\boldsymbol{H}_{\varphi}\). 
    We apply two node embeddings to obtain \( k \)-NNs of \(v_i\), i.e., \(\mathcal{B}_i^{\theta}\), \(\mathcal{B}_i^{\varphi}\),
    intersects which with the topological \( k \)-NNs \(\mathcal{T}_i\) to obtain positive and negative pairs for node \(v_i\) separately, i.e., \(\mathcal{P}_i\) and \(\mathcal{N}_i\). 
    Finally, we employ the contrastive loss to achieve that the anchor nodes close to positive pairs and far from negative pairs.
    }
    \label{fig:model}
\end{figure*}

\section{Method}
In this section, we provide a detailed description of GTCA, including graph encoders, node sampling and contrastive loss function. 
Figure \ref{fig:model} shows the model architecture of GTCA.

\subsection{Preliminaries}
Let \(\mathcal{G} = (\mathcal{V}, \mathcal{E})\) denote a graph, where \(\mathcal{V} = \{ v_1, \cdots, v_N\}\), \(\mathcal{E} \subseteq \mathcal V \times \mathcal V\) represent the node set and the edge set respectively.
We denote the feature matrix and the adjacency matrix as \(\bm{X} \in \mathbb{R}^{N \times F}\) and \(\bm{A} \in \{0,1\}^{N \times N}\), where \(\bm{x}_i \in \mathbb{R}^{F}\) is the feature of \(v_i\), and \(\bm{A}_{ij} = 1\) if \((v_i, v_j) \in \mathcal{E}\).
We aim to learn a GNN encoder \( f_\theta (\bm{X}, \bm{A}) \in \mathbb{R}^{N \times D} \) and a Transformer encoder \( g_\varphi (\bm{X}, \bm{A}) \in \mathbb{R}^{N \times D} \). 
They take the node features and structures as input, and generate node embeddings in low dimensionality, i.e., \( D \ll F \).
We denote \( \bm{H}_\theta = f_\theta (\bm{X}, \bm{A}) \) as GNN node representations and \( \bm{H}_\varphi = g_\varphi (\bm{X}, \bm{A}) \) as Transformer node representations,
where \( \bm{h}_i^\theta \) is GNN embedding of node \( v_i \), and \( \bm{h}_i^\varphi \) is the Transformer embedding of node \( v_i \) .

\subsection{Graph Encoders}
We employ GCN~\cite{kipf2016semi} as an encoder to capture the local structural information of the nodes. 
GCN utilizes a series of graph convolution layers to aggregate information from the neighbors.
We update each layer according to the following equation:

\begin{equation}
    \boldsymbol{H}^{(l+1)} = \sigma \left( \hat{\boldsymbol{A}} \, \boldsymbol{H}^{(l)} \, \boldsymbol{W}^{(l)} \right)
    \end{equation}
where \(\boldsymbol{H}^{(l)}\) denotes the feature matrix at layer \(l\), \(\boldsymbol{H}^{(0)}\) = \(\bm{X}\), \(\hat{\boldsymbol{A}}\) is the normalized adjacency matrix with self-loops, \(\boldsymbol{W}^{(l)}\) represents the learnable weight matrix for layer \(l\), and \(\sigma\) is a non-linear activation function, i.e., ReLU.
We apply a 2-layer GCN to obtain \(\boldsymbol{H}_{\theta}\). 

Meanwhile, we utilize the NodeFormer~\cite{nodeformer} as the Transformer encoder with the following equation:
\begin{equation}\label{eqn-attn-gumbel-final}
    \mathbf h_i^{(l+1)} \approx \sum_{j=1}^N \frac{\phi(\mathbf q_i / \sqrt{\tau})^\top \phi(\mathbf k_j / \sqrt{\tau}) e^{g_j/\tau}}{\sum_{w=1}^N \phi(\mathbf q_i / \sqrt{\tau})^\top \phi(\mathbf k_w / \sqrt{\tau}) e^{g_w/\tau}} \cdot \mathbf v_j  
\end{equation} 
where \( \phi (\cdot) \) denotes a kernel function and $\mathbf q_i = \boldsymbol{W}_Q^{(l)} \mathbf h_i^{(l)}$, $\mathbf k_j = \boldsymbol{W}_K^{(l)} \mathbf h_j^{(l)}$ and $\mathbf v_j = \boldsymbol{W}_V^{(l)} \mathbf h_j^{(l)}$.
We can obtain \(\boldsymbol{H}_{\varphi}\) with NodeFormer.

Contrastive learning is an important tool for multi-view learning. 
Inspired by~\cite{huang2021makes}, we present our finding with respect to multi-view learning.
\begin{theorem}\label{thm:mn-modality}
    Let \(\mathcal{G} = (\mathcal{V}, \mathcal{E})\) be a graph dataset of $N$ nodes drawn i.i.d. according to an unknown distribution $\mathcal{D} .$ 
    Let $\mathcal{M},\mathcal{N}$ be two distinct subsets of $[K]$, where $K$ is the number of multiple views. Assume empirical risk 
    minimizers $(\hat{h}_{\mathcal{M}},\hat{g}_{\mathcal{M}})$ and $(\hat{h}_{\mathcal{N}},\hat{g}_{\mathcal{N}})$, training with the $\mathcal{M}$ and $\mathcal{N}$  views separately. Then, 
           for all $1>\delta>0,$ with probability at least $1-\frac{\delta}{2}$:
            \begin{align}
                &r\left(\hat{h}_{\mathcal{M}}\circ \hat{g}_{\mathcal{M}}\right) - r\left(\hat{h}_{\mathcal{N}}\circ \hat{g}_{\mathcal{N}}\right) \leq  \nonumber \\
                &\gamma_{\mathcal{G}}(\mathcal{M},\mathcal{N}) + 8L\mathfrak{R}_{N}( \mathcal{H}\circ\mathcal{G}_{\mathcal{M}}) + \frac{4C}{\sqrt{N}} + 2C\sqrt{\frac{2 \ln (2 / \delta)}{N}} \label{mg}
            \end{align}
 
             where 
       \begin{align}&\gamma_{\mathcal{G}}(\mathcal{M},\mathcal{N})\triangleq \eta(\hat{g}_{\mathcal{M}})-\eta(\hat{g}_{\mathcal{N}})\label{gamma} \qquad\Box
       \end{align}
       \end{theorem}   
       
\noindent \textbf{Remark.} 
First, $\gamma_{\mathcal{G}}(\mathcal{M},\mathcal{N})$ of $(\ref{gamma})$ trades off the quality between latent representations learning from $\mathcal{M}$ and $\mathcal{N}$ views with 
respect to the graph dataset $\mathcal{G}$. Theorem \ref{thm:mn-modality}  bounds the difference of population risk training with two different subsets of views with 
$\gamma_{\mathcal{G}}(\mathcal{M},\mathcal{N})$, which validates our intuition that more views is superior.
    Second, for the commonly used function classes, Radamacher complexity for a node of size $N$, $\mathfrak{R}_{N}(\mathcal{F})$ is generally bounded by 
     $\sqrt{C(\mathcal{F}) / N}$, where $C(\mathcal{F})$ represents the intrinsic property of function class $\mathcal{F}$. 
    Third, $(\ref{mg})$ can be written as $\gamma_{\mathcal{G}}(\mathcal{M},\mathcal{N})+\mathcal{O}(\sqrt{\frac{1}{N}})$ in order terms. This indicates that as the number of node increases,
    the performance with different views mainly depends on its latent representation quality. 

\subsection{Node Sampling}
We compute the cosine similarity matrix of \(\boldsymbol{H}_{\theta}\) and \(\boldsymbol{H}_{\varphi}\) separately to obtain the \( k \)-NN sets.
Specifically, we obtain the \(k\)-NN node set \(\mathcal{B}_i^{\theta}\) of node \(v_i\) with the cosine similarity between \( \bm{h}_i^\theta \) and \(\boldsymbol{H}_{\theta}\). 
Similarly, we obtain the \(k\)-NN node set \(\mathcal{B}_i^{\varphi}\) for node \(v_i\) with the cosine similarity between \( \bm{h}_i^\varphi \) and \(\boldsymbol{H}_{\varphi}\).
Furthermore, we compute the topological \(k\)-NN matrix via nodes sorting according to the number of hops from other nodes to node \(v_i\) in ascending order, resulting in the topological \(k\)-NN set \(\mathcal{T}_i\) for node \(v_i\).
Finally, we obtain the positive set \(\mathcal{P}_i\) for node \(v_i\) from the intersection of \(\mathcal{B}_i^{\theta}\), \(\mathcal{B}_i^{\varphi}\), and \(\mathcal{T}_i\):

\begin{equation}
    \label{positive}
    \mathcal{P}_i = \mathcal{B}_i^{\theta} \cap \mathcal{B}_i^{\varphi} \cap \mathcal{T}_i
    \end{equation}

Meanwhile, we treat all other nodes as negative pairs.
\begin{equation}
    \label{negative}
    \mathcal{N}_i = \mathcal{V} \setminus (\mathcal{P}_i \cup \{v_i\})
    \end{equation}

Compared with existing augmentation-based contrastive learning models, GTCA has several advantages. 
First, GTCA discards random or heuristic augmentation strategies which may disturb the graph structure. 
Second, in contrast to existing GCL methods that utilize shared-weight GNNs, GTCA uses GNN and Transformer encoders to learn diverse graph feature representations.

\subsection{Contrastive Loss Function}
According to \cite{InfoNCE}, the InfoNCE is a lower bound of the true Mutual Information (MI).
Recent GCL methods use InfoNCE as the loss function \cite{GRACE,spectral,architecture,provable}.
Figure~\ref{fig:InfoNCE} gives an example of the InfoNCE loss. 
For a given anchor node, InfoNCE uses only one node to generate a positive pair and utilizes all the other nodes to generate negative pairs.
This may lead to ignorance of the potential information of similar nodes.
\begin{figure}[htbp]
    \centering
    \includegraphics[height=0.7\columnwidth]{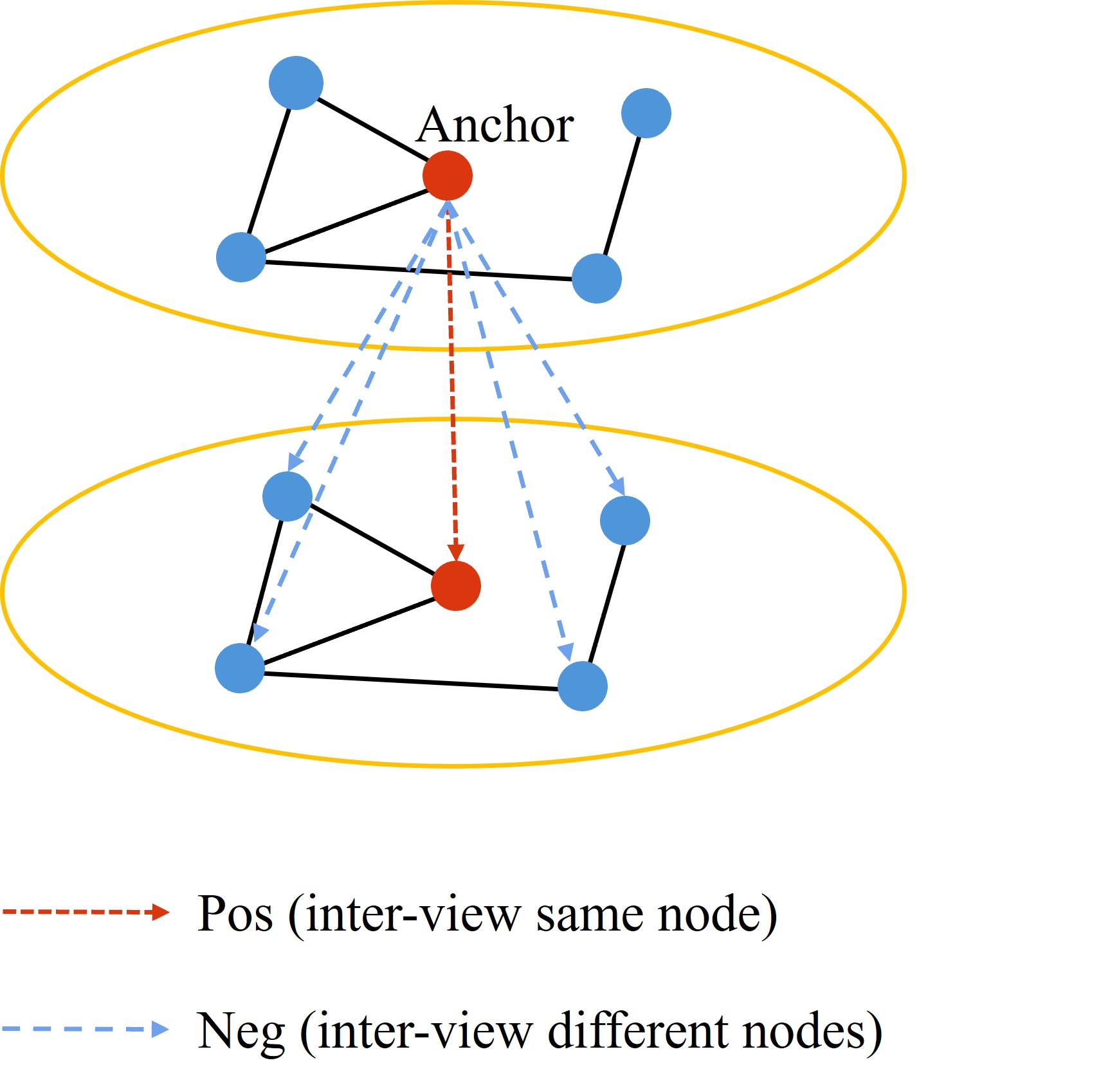} 
    \caption{The red nodes denote the positive pair and the blue dotted lines with arrows are negative pairs of the anchor.}
    \label{fig:InfoNCE}
\end{figure}

To address this issue, we design a novel contrastive loss function. 
Let \( \bm{h}_i^\theta \) and \( \bm{h}_i^\varphi \) denote the \( \ell_2 \)-normalized GCN and NodeFormer embeddings of \( v_i \) respectively.
Select \( \bm{h}_i^\theta \) as the anchor node, its positives come from three sources: 
(1) its NodeFormer embedding \( \bm{h}_i^\varphi \); 
(2) GCN embeddings $\{ \bm{h}_j^\theta \mid v_j \in \mathcal{P}_i \}$; 
(3) NodeFormer embeddings $\{ \bm{h}_j^\varphi \mid v_j \in \mathcal{P}_i \}$.
That is, the total number of positive pairs associated with node \( v_i \) is \( 2|\mathcal{P}_i| + 1 \). 
Ultimately, the contrastive loss function is as follows:

\begin{equation}
    \label{loss_i}
    \resizebox{\linewidth}{!}{$
    \ell(\bm{h}_i^\theta) = -\log  \frac{e^{s(\bm{h}_i^\theta, \bm{h}_i^\varphi) / \tau} + \sum_{j \in \mathbf{P}_i} (e^{s(\bm{h}_i^\theta, \bm{h}_j^\theta) / \tau} +  e^{s(\bm{h}_i^\theta, \bm{h}_j^\varphi) / \tau})}{e^{s(\bm{h}_i^\theta, \bm{h}_i^\varphi) / \tau} + \sum_{j \ne i} (e^{s(\bm{h}_i^\theta, \bm{h}_j^\theta) / \tau} +  e^{s(\bm{h}_i^\theta, \bm{h}_j^\varphi) / \tau})} 
   $}
    \end{equation}
where \( s(\cdot, \cdot) \) is the cosine similarity and \( \tau \) is a temperature parameter.
For a given query node $v_i\in\mathcal{V}$, we compute the cosine similarity as follows:
\begin{eqnarray}
\small
s(\bm{h}_i^\theta, \bm{h}_i^\varphi) = \frac{\bm{h}_i^\theta \cdot \bm{h}_i^\varphi}{\|\bm{h}_i^\theta\|\|\bm{h}_i^\varphi\|}, \forall v_j\in\mathcal{V}
\end{eqnarray}
The last two terms of the denominator of Equation~(\ref{loss_i}) are decomposed as:

\begin{equation}
    \resizebox{\linewidth}{!}{$
\sum_{j \ne i} e^{s(\bm{h}_i^\theta, \bm{h}_j^\theta) / \tau} =  
\sum_{j \in \mathbf{P}_i} e^{s(\bm{h}_i^\theta, \bm{h}_j^\theta) / \tau} + 
\sum_{j \in \mathbf{N}_i} e^{s(\bm{h}_i^\theta, \bm{h}_j^\theta) / \tau}
    $}
    \end{equation}


\begin{equation}
    \resizebox{\linewidth}{!}{$
\sum_{j \ne i} e^{s(\bm{h}_i^\theta, \bm{h}_j^\varphi) / \tau} =  
\sum_{j \in \mathbf{P}_i} e^{s(\bm{h}_i^\theta, \bm{h}_j^\varphi) / \tau} + 
\sum_{j \in \mathbf{N}_i} e^{s(\bm{h}_i^\theta, \bm{h}_j^\varphi) / \tau}
$}
\end{equation}

The minimization of Equation~(\ref{loss_i}) will pull positive pairs closer and push negative pairs away.
Since the two views are symmetric, we define the loss for the feature embedding \( \bm{h}_i^\varphi \) corresponding to node \( v_i \) in the other view similarly to Equation~(\ref{loss_i}).
Then, the final loss function is as follows:
\begin{equation}
    \label{Loss}
\mathcal{L} = \frac{1}{2N} \sum_{i=1}^{N} \left[ \ell(\bm{h}_i^\theta) + \ell(\bm{h}_i^\varphi) \right]
\end{equation}

\begin{proposition}\label{prop:prop}
    Given a graph \(\mathcal{G} = (\mathcal{V}, \mathcal{E})\), encoders $f_\theta$, $g_\varphi$ and 
    a contrastive loss function $\mathcal{L}$ defined in Equation~(\ref{Loss}).
    \(\mathcal{B}_i^{\theta}\), \(\mathcal{B}_i^{\varphi}\), and \(\mathcal{T}_i\) are \( k \)-NNs of node \(v_i\), \(i = 1, 2, \cdots, N \) under two node representation views and one topology structure view.
    Equation~(\ref{positive}) generates the most trustworthy positive pairs for GCL. 
       \end{proposition}

\begin{algorithm}[H]
\caption{GTCA}
\label{alg:algorithm}
\textbf{Input}: The adjacency matrix \(\bm{A}\), the feature matrix \(\bm{X}\), and the number of training epochs \(J\)\\
\textbf{Output}: Feature matrix \( \boldsymbol{H} \)
\begin{algorithmic}[1] 
\FOR{epoch in 1 to \(J\)}
    \STATE Generate GNN embeddings \(\boldsymbol{H}_{\theta}\) and NodeFormer embeddings \(\boldsymbol{H}_{\varphi}\) with GNN encoder \(f_\theta\) , NodeFormer encoder \(g_\varphi\), adjacency matrix \(\bm{A}\) and feature matrix \(\bm{X}\);  
    \STATE Generate GNN \(k\)-NN node set \(\mathcal{B}_i^{\theta}\), Nodeformer \(k\)-NN node set \(\mathcal{B}_i^{\varphi}\) and topological \(k\)-NN set \(\mathcal{T}_i\) with \(\boldsymbol{H}_{\theta}\) , \(\boldsymbol{H}_{\varphi}\) and adjacency matrix \(\bm{A}\);
    \STATE Calculate positive pairs \( \mathcal{P}_i \) and negative pairs \( \mathcal{N}_i \), \(i = 1, \cdots, N \) with Equation~(\ref{positive}) and Equation~(\ref{negative});
    \STATE Compute loss \(\mathcal{L}\) with Equation~(\ref{Loss});
    \STATE Apply gradient descent to minimize \(\mathcal{L}\) and update parameters; 
\ENDFOR
\STATE Calculate the final output feature matrix with Equation~(\ref{output});
\STATE \textbf{teturn} \(\boldsymbol{H}\) for downstream tasks;
\end{algorithmic}
\end{algorithm}

\begin{proof}
    It is obvious that the intersection set contains the minimum number of positive pairs with specific \(\mathcal{B}_i^{\theta}\), \(\mathcal{B}_i^{\varphi}\), and \(\mathcal{T}_i\).
    The numerator of  Equation~(\ref{loss_i}) sums the exponential calculation with respect to \(\mathcal{P}_i\), 
    whereas the denominator sums the exponential calculation with respect to both \(\mathcal{P}_i\) and \(\mathcal{N}_i\).
    Therefore, when calculating \(\mathcal{P}_i\) with Equation~(\ref{positive}),
    we get the lower bound of the numerator, and also the upper bound of the contrastive loss in Equation~(\ref{loss_i}).
    When we minimize the upper bound of the loss in Equation~(\ref{Loss}), trustworthy GCL can be achieved. 
    \end{proof}
At each training epoch, GTCA first generates two node feature representation matrices \(\boldsymbol{H}_{\theta}\) and \(\boldsymbol{H}_{\varphi}\) with GCN encoder \(f_{\theta}\) and NodeFormer encoder \(g_{\varphi}\), respectively. 
Then, we calculate the positive and negative pairs with Equation~(\ref{positive}) and Equation~(\ref{negative}).
Finally, we minimize the objective in Equation~(\ref{Loss}) to update the parameters of \(f_{\theta}\) and \(g_{\varphi}\). 
We obtain two trained feature matrices, \( \boldsymbol{H}_{\theta}^{\prime} \) and \( \boldsymbol{H}_{\varphi}^{\prime} \). 
The final output feature matrix \( \boldsymbol{H} \) is obtained with the normalized feature matrices \( \boldsymbol{H}_{\theta}^{\prime} \) and \( \boldsymbol{H}_{\varphi}^{\prime} \) according to a weight parameter \( \lambda \):

\begin{equation}
    \label{output}
    \boldsymbol{H} = \lambda \cdot \boldsymbol{H}_{\theta}^{\prime} + (1 - \lambda) \cdot \boldsymbol{H}_{\varphi}^{\prime}
\end{equation}
The final output feature matrix \( \boldsymbol{H} \), a linear combination of \(\boldsymbol{H}_{\theta}^{\prime}\) and \(\boldsymbol{H}_{\varphi}^{\prime}\), is applied to downstream node classification tasks.
\(\lambda\) is a tunable weight parameter.
Algorithm~\ref{alg:algorithm} summarizes the overall procedure of GTCA.

\section{Experiments}
\subsection{Datasets}
To validate the effectiveness of the GTCA method, we perform extensive experiments on 5 benchmark datasets for node classification
including a commonly used citation network, i.e., Cora~\cite{sen2008collective}, a reference network constructed based on
Wikipedia, i.e., Wiki-CS~\cite{wiki}, a co-authorship network, i.e., Coauthor-CS~\cite{shchur2018pitfalls}, 
and two product co-purchase networks, i.e., Amazon-Computers and Amazon-Photo~\cite{shchur2018pitfalls}. 
We list the detailed statistics of these datasets in Table~\ref{tab:datasets}.

\begin{table}[h]
    \centering
    \small
    \setlength{\tabcolsep}{4pt} 
    \renewcommand{\arraystretch}{1.2} 
    \begin{tabular}{cccccc}
        \hline
        Datasets & \# Nodes & \# Edges & \# Features & \# Labels \\
        \hline
        Cora          & 2,708  & 5429 & 1,433  & 7 \\
        Wiki-CS      & 11,701  &  216,123   &  300  & 10 \\
        Coauthor-CS   & 18,333 & 81894 & 6,805  & 15 \\
        Amazon-Computers  & 13,752 & 245,861 &  767  & 10 \\
        Amazon-Photo  & 7,650  & 119081 & 745    & 8 \\
        \hline
    \end{tabular}
    \caption{Statistics of the datasets.}
    \label{tab:datasets}
\end{table}

\subsection{Baselines}
We compare GTCA with 12 state-of-the-art methods, including 2 semi-supervised GNNs, i.e., GCN~\cite{kipf2016semi} and
GAT~\cite{GAT}, 2 semi-supervised GCL methods, i.e., CGPN~\cite{CGPN} and CG3~\cite{CG3}, and 8 self-supervised GCL methods, i.e., DGI~\cite{DGI}, GMI~\cite{GMI}, MVGRL~\cite{MVGRL}, GRACE~\cite{GRACE},
GCA~\cite{GCA}, SUGRL~\cite{SUGRL}, AFGRL~\cite{AFGRL} and NCLA~\cite{NCLA}.

\subsection{Experimental Settings}
We employ the above methods for node classification. 
For Cora dataset, we follow \cite{labelset3} to randomly select 20 nodes per class for training, 500 nodes for validation, and the remaining nodes for testing. 
For Wiki-CS, Coauthor-CS, Amazon-Computers and Amazon-Photo datasets, we follow \cite{labelset4} to randomly select 20 nodes per class for training, 30 nodes per class for validation, and the remaining nodes for testing. 
We perform 20 random splits of training, validation, and testing on each dataset and report the average performance of all algorithms across these splits.
All experiments are implemented in PyTorch and conducted on a server with NVIDIA GeForce 3090 (24GB memory each). 
Table~\ref{tab:hyperparameter} shows the hyperparameter settings of GTCA on 5 datasets.

\begin{table}[H]
    \centering
    \setlength{\tabcolsep}{9pt} 
    \renewcommand{\arraystretch}{1.2} 
    \begin{tabular}{ccccc}
        \hline
        Datasets  &  \( k \) & \( E \) &   $\lambda$  &   \( lr \) \\
        \hline
        Cora          & 520  & 440  & 0.7  & 0.005 \\
        Wiki-CS        & 500 & 400  & 0.8    & 0.001 \\
        Coauthor-CS   & 240 & 420 & 0.4  & 0.001 \\
        Amazon-Computers      & 550  & 512   & 0.8  & 0.001 \\
        Amazon-Photo  & 510  & 512 & 0.7    & 0.001 \\
        \hline
    \end{tabular}
    \caption{Hyperparameter settings of GTCA on 5 datasets. \( lr \) is the learning rate.}
    \label{tab:hyperparameter}
\end{table} 

\begin{table*}[t]
    \centering
    \renewcommand{\arraystretch}{1.1} 
    \begin{tabular}{ccccccccc}
    \hline
    Methods & Cora & Wiki-CS & Coauthor-CS & Amazon-Computers & Amazon-Photo \\
    \hline
    GCN & 79.6±1.8 & 67.3±1.5 & 90.0±0.6 & 76.4±1.8 & 86.3±1.6 \\
    GAT & 81.2±1.6 & 68.6±1.9 & 90.9±0.7 & 77.9±1.8 & 86.5±2.1 \\
    CGPN & 74.0±1.7 & 66.1±2.1 & 83.5±1.4 & 74.7±1.3 & 84.1±1.5 \\
    CG3 & 80.6±1.6 & 68.0±1.5 & 90.6±1.0 & 77.8±1.7 & 89.4±1.9 \\
    DGI & 82.1±1.3 & 69.1±1.4 & \underline{92.0±0.5} & 78.8±1.1 & 83.5±1.2 \\
    GMI & 79.4±1.2 & 67.8±1.8 & 88.5±0.8 & 76.1±1.2 & 86.7±1.5 \\
    MVGRL & \underline{82.4±1.5} & 69.2±1.2 & 91.5±0.6 & 78.7±1.7 & 89.7±1.2 \\
    GRACE & 79.6±1.4 & 67.8±1.4 & 90.0±0.7 & 76.8±1.7 & 87.9±1.4 \\
    GCA & 79.0±1.4 & 67.6±1.3 & 90.9±1.1 & 76.9±1.4 & 87.0±1.9 \\
    SUGRL & 81.3±1.2 & 68.7±1.1 & 91.2±0.9 & 78.2±1.2 & \underline{90.5±1.9} \\
    AFGRL & 78.6±1.3 & 68.0±1.7 & 91.4±0.6 & 77.7±1.1 & 89.2±1.1 \\
    NCLA & 82.2±1.6 & \textbf{70.3±1.7} & 91.5±0.7 & \textbf{79.8±1.5} & 90.2±1.3 \\
    GTCA & \textbf{82.5±1.3} & \underline{69.7±1.5} & \textbf{92.5±0.6} & \underline{79.2±1.4} & \textbf{90.5±1.2} \\
    \hline
    \end{tabular}
    \caption{Node classification accuracy (\%) comparison on 5 datasets.}
    \label{tab:performance_comparison}
\end{table*}

\subsection{Node Classification Results}
Table~\ref{tab:performance_comparison} shows node classification accuracy on 5 benchmark graph datasets. 
On the whole, GTCA demonstrates superior performance across all 5 datasets. 
GTCA ranks first on three datasets and ranks second on the other two datasets.

The remarkable performance of GTCA can be attributed to two key aspects. 
First, GTCA introduces an augmentation-free strategy, which avoids the potential risk of disturbing the underlying semantics of graphs. 
Second, GTCA employs a novel sampling strategy to generate trustworthy positive pairs and negative pairs.

\subsection{Ablation Study}
Table~\ref{tab:Ablation} presents the ablation study results of GTCA and its variants without \(\mathcal{T}_i\) and with different graph encoders including GNNs $f_\theta, f_\varphi$ and NodeFormers $g_\theta, g_\varphi$. We observe that there is an obvious decline in performance without topology structure view on 5 datasets. 
This indicates that topological information is essential for enhancing node classification accuracy. In terms of graph encoders, compared with GTCA, the use of either GCN-GCN or NodeFormer-NodeFormer results in inferior performance, which validates the reasonability of the module design. Meanwhile, we can observe that 
\(\text{GTCA}\_{g_\theta, g_\varphi}\) ranks recond on the Coauthor-CS dataset, which has the largest number of nodes among the 5 datasets, due to the fact that Transformer-based models can effectively capture long-range dependencies.

\begin{table}[H]
    \centering
    \small 
    \setlength{\tabcolsep}{2pt} 
    \renewcommand{\arraystretch}{1} 
    \begin{tabular}{lccccc}
        \hline
        & \makecell{Cora} & \makecell{Wiki-CS} & \makecell{Coauthor-\\CS} & \makecell{Amazon-\\Computers} & \makecell{Amazon-\\Photo} \\
        \hline
        GTCA w/o \(\mathcal{T}_i\) & 79.8 & 69.0 & 92.0 & 75.7 & 86.1 \\
        \text{GTCA}\_{$f_\theta, f_\varphi$} & 80.4 & 69.1 & 90.0 & 78.1 & 90.1 \\
        \text{GTCA}\_{$g_\theta, g_\varphi$} & 77.6 & 68.8 & 92.1 & 76.9 & 88.8 \\
        GTCA & \textbf{82.5} & \textbf{69.7} & \textbf{92.5} & \textbf{79.2} & \textbf{90.5} \\
        \hline
    \end{tabular}
    \caption{Node classification accuracy (\%) of different components on 5 datasets.}
    \label{tab:Ablation}
\end{table}

\subsection{Graph Encoders Analysis}
Figure~\ref{fig:ratio} shows that with a single graph encoder, correct ratio of positive pairs is lower than 40\%.
With both GCN and NodeFormer as graph encoders, the correct ratio of positive pairs varies from 30\% to 60\%.
When GTCA exploits intersection of the \( k \)-NN neighborhood of GCN, Nodeformer node representation and topology structure views,
the correct ratio of positive pairs exceeds 80\%.  
This highlights the soundness of the current graph encoder design. 

\begin{figure}[h]
    \centering
    \includegraphics[width=1\columnwidth]{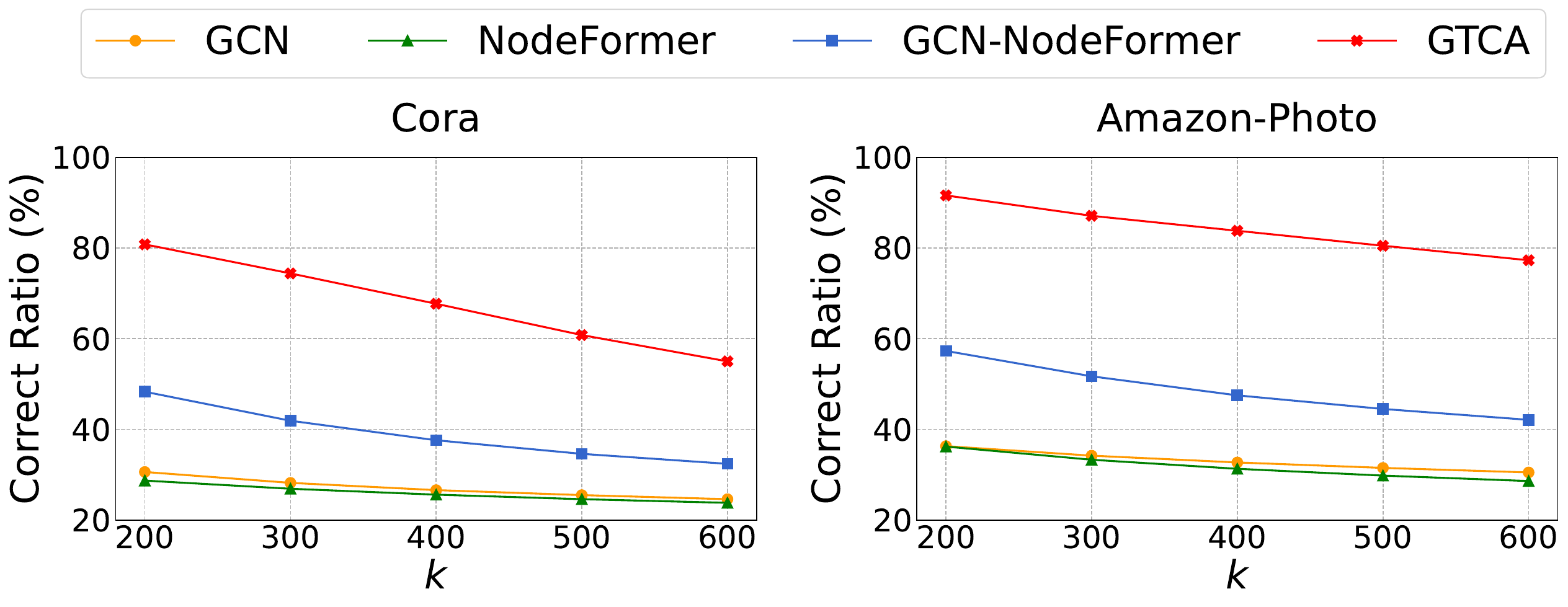}
    \caption{Correct ratio of positive pairs with various \( k \) values on Cora and Amazon-Photo datasets.}
    \label{fig:ratio}
    \vspace{-2ex} 
\end{figure}

\begin{figure*}[t]
    \centering
    \includegraphics[width=1.3\columnwidth]{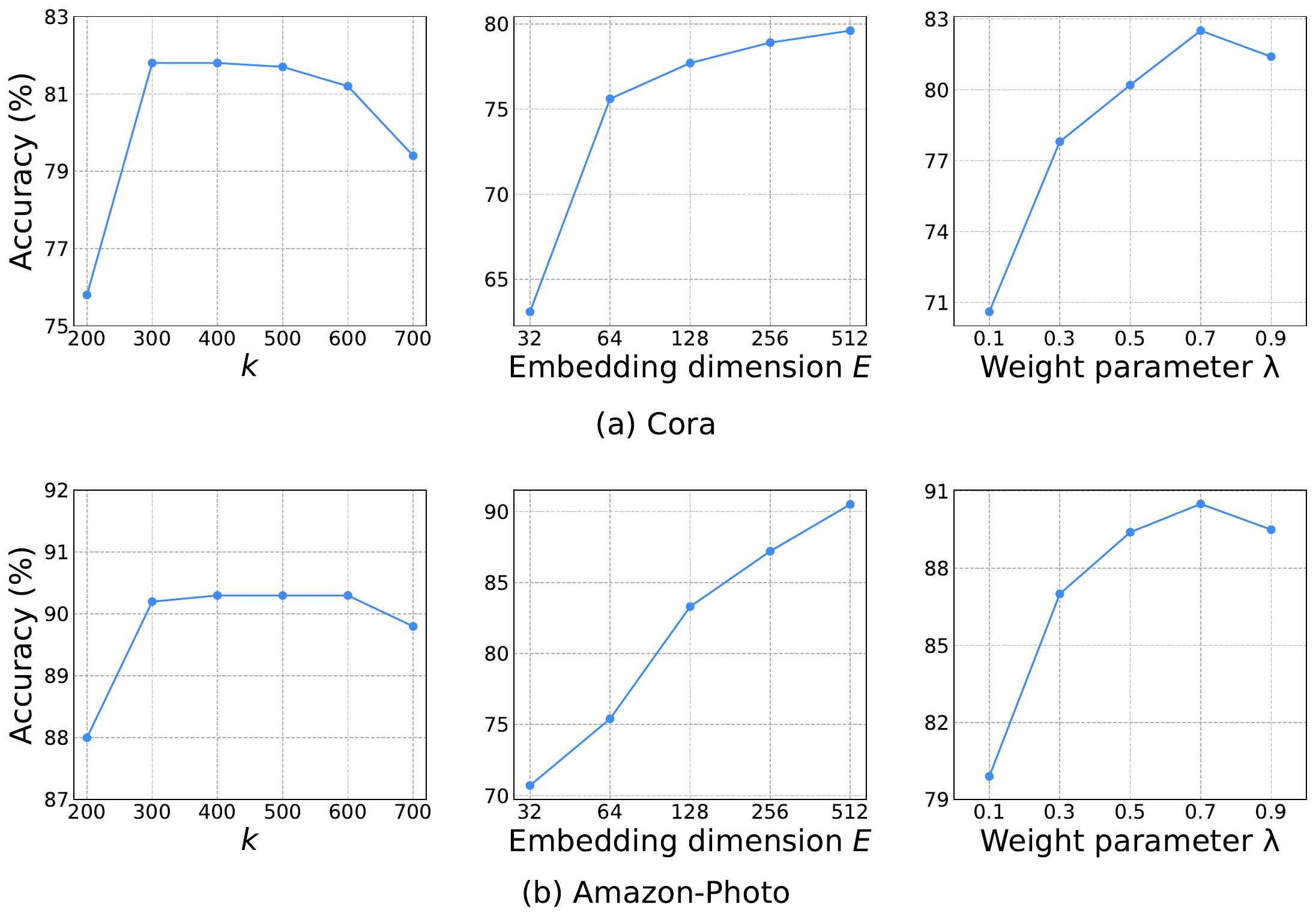} 
    \caption{Accuracy vs. hyperparameters \( k \), \( E \) and $\lambda$ on Cora and Amazon-Photo datasets.}
    \label{fig:cora-photo}
\end{figure*}

\begin{figure*}[h]  
    \centering
    \includegraphics[width=1.8\columnwidth]{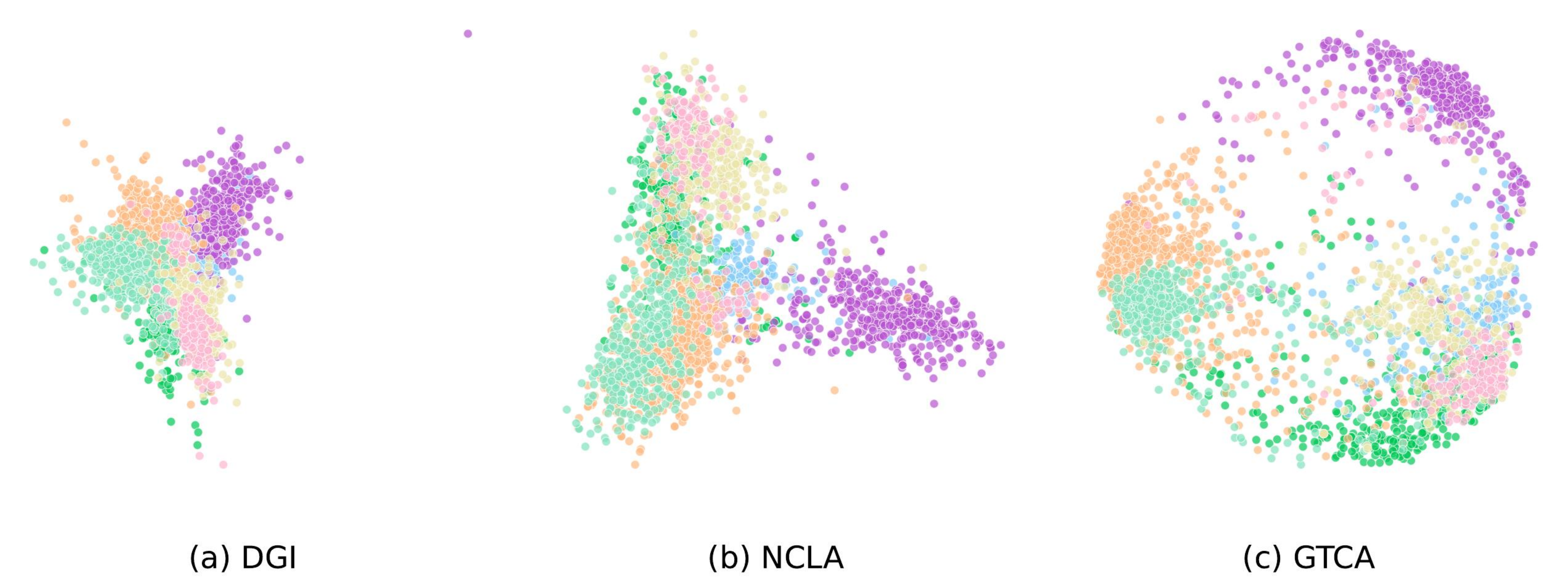}
    \caption{Visualization of DGI, NCLA and GTCA embeddings on Cora dataset with PCA.}
    \label{fig:tsne}
\end{figure*}

\subsection{Sensitivity Analysis}
Figure~\ref{fig:cora-photo} illustrates the node classification accuracy of GTCA on Cora and Amazon-Photo datasets with various hyperparameters, respectively. 
The optimal range of \( k \) is between 300 and 600.
When \( k \) is too small, GTCA has inferior performance due to a scarcity of positive pairs. 
Conversely, a large \( k \) introduces too much redundant information.
Thus, it is necessary to search for a proper \( k \). 
In addition, increasing embedding dimension \( E \) improves classification accuracy. 
Moreover, the hyperparameter $\lambda$ is crucial for model performance. 
Initially, an increase in $\lambda$ improves node classification accuracy. 
However, when $\lambda$ exceeds 0.7, the accuracy will decrease.

\subsection{Embeddings Visualization}
To provide a more intuitive presentation of the node embeddings, we utilize PCA~\cite{PCA} to visualize the node embeddings of DGI in Figure~\ref{fig:tsne} (a), NCLA in Figure~\ref{fig:tsne} (b) and GTCA in Figure~\ref{fig:tsne} (c). 
Different colors represent different categories.
Compared with DGI and NCLA, GTCA can distinguish different classes of nodes much more effectively.

\section{Conclusion}
Most of the existing GCL methods utilize graph augmentation strategies, which may perturb the underlying semantics of graphs.
Furthermore, they utilize GNNs as graph encoders, which inevitably results in the occurrence of over-smoothing and over-squashing issues.
To address these issues, we propose GNN-Transformer Cooperative Architecture for Trustworthy Graph Contrastive Learning (GTCA).
GTCA uses GCN and NodeFormer as encoders to generate node representation views. 
In addition, it utilizes topological property of graphs to generate the topology structure views. 
Theoretical analysis and experimental results demonstrate the effectiveness of GTCA.

While GTCA has been proved effective, it still faces challenges in terms of computational complexity.
Specifically, its quadratic complexity poses a significant computational burden when handling large-scale graph data.
In future work, we aim to explore more efficient techniques, 
such as parallel processing and distributed systems.

\section{Acknowledgements}
This work is supported by the National Science and
Technology Major Project (2020AAA0106102) and National
Natural Science Foundation of China (No.62376142,
U21A20473, 62306205).
\bibliography{aaai25}

\end{document}